\documentclass[letterpaper, 10 pt, conference]{ieeeconf}  
\IEEEoverridecommandlockouts             
\usepackage{bm}
\usepackage{enumerate}
\usepackage{commath}
\usepackage{graphicx}
\graphicspath{{Pictures/}}
\usepackage{amsmath}
\usepackage[font=small]{caption}
\usepackage[font=small]{subcaption}
\usepackage{breqn}
\usepackage{cite}
\usepackage[table]{xcolor}
\usepackage{booktabs}
\usepackage{empheq}
\usepackage{amsfonts}
\usepackage{amssymb}

\usepackage{amsthm}
\usepackage{hyperref}
\usepackage{xcolor}
\usepackage{mathrsfs}
\usepackage{accents}
\usepackage{thmtools}
\usepackage{thm-restate}
\usepackage{float}
\usepackage{xcolor}
\usepackage{comment}
\usepackage[normalem]{ulem}
\usepackage{float}

\theoremstyle{plain}
\newtheorem{lemma}{Lemma}
\newtheorem{definition}{Definition}

\newtheorem{property}{Property}
\newtheorem{cor}{Corollary}

\newtheorem{remark}{Remark}
\newtheorem{example}{Example}
\newtheorem{problem}{Problem}
\newtheorem*{problem*}{Problem}
\newtheorem*{theorem*}{Theorem}
\newtheorem{assumption*}{Assumption}

\declaretheorem[name=Theorem]{thm}

\newcommand{\redtext}[1]{{\color{red}#1}}

\newcommand{\myvar}[1]{\bm{#1}}
\newcommand{\tildevar}[1]{\tilde{\bm{#1}}}

\newcommand{\myvardot}[1]{\dot{\myvar{#1}}}

\newcommand{\myset}[1]{\mathcal{#1}}
\newcommand{\mysetbound}[1]{\partial \myset{#1}}
\newcommand{\mysetint}[1]{\mathring{\myset{#1}}}

\title{\LARGE \bf
Safe, Passive Control for Mechanical Systems with Application to Physical Human-Robot Interactions
}

 \author{ Wenceslao Shaw Cortez, Christos K. Verginis, and Dimos V. Dimarogonas 
\thanks{This work was supported by the Swedish Research Council (VR),
 the Swedish Foundation for Strategic Research (SSF), the Knut and Alice
Wallenberg Foundation (KAW) and  the H2020 ERC Consolidator Grant LEAFHOUND.
The authors are with the School of EECS, Royal Institute of Technology (KTH), 100 44 Stockholm, Sweden (Email: 
       {\tt\small wencsc, cverginis, dimos@kth.se}).}
}

\begin{document}

\maketitle
\thispagestyle{plain}
\pagestyle{plain}

\begin{abstract}

In this paper, we propose a novel safe, passive, and robust control law for mechanical systems. The proposed approach addresses safety from a physical human-robot interaction perspective, where a robot must not only stay inside a pre-defined region, but respect velocity constraints and ensure passivity with respect to external perturbations that may arise from a human or the environment. The proposed control is written in closed-form, behaves well even during singular configurations, and allows any nominal control law to be applied inside the operating region as long as the safety requirements (e.g., velocity) are adhered to. The proposed method is implemented on a 6-DOF robot to demonstrate its effectiveness during a physical human-robot interaction task.

\end{abstract}

\section{Introduction}

Much of today's robotics research aims to place humans in the vicinity of, and in contact with, robots for cooperative or co-existing tasks \cite{Luca2012,Cherubini2016}. This cooperative behaviour exploits the abilities of the human and robot to synergistically complete a task. One obvious requirement in such a physical human-robot interaction (pHRI) setting is to ensure safety. 

Safety in the context of pHRI has different definitions \cite{Zacharaki2020}. Many methods share the philosophy that ``passivity implies safety" \cite{Papageorgiou2020,Music2018, Tadele2014}. Passivity is a desirable property because it is \textit{necessary} to ensure a stable interaction with any unknown environment \cite{Schaft2017,Stramigioli2015}. However, passivity alone may not be ``safe" according to industry standards where machines must satisfy velocity/power/force constraints and stay within an operating region \cite{isosafety2011}. Even methods that focus on ISO (International Organization for Standardization) standards or passivity may succumb to ``ill-posedness", i.e., they may fail under singular/non-full rank Jacobians  \cite{Navarro2016, Papageorgiou2020,Music2018}. Control laws should be robust, i.e., provide asymptotic stability guarantees \cite{Tadele2014} and be well-defined in the robot workspace to ensure safety. For a full survey of safe pHRI see \cite{Zacharaki2020,Tadele2014}. 

To address safety, we focus on the concept of ``novel, robust, and generalizable safety methods" and compliance of the robot to human actions \cite{Zacharaki2020}. We aim to augment existing work by designing a control law with respect to a safe operating region in which a human and robot may be in close proximity or in contact. In this region, the robot must obey strict velocity constraints to minimize injury  \cite{isosafety2011}. For generality, we allow \textit{any} existing pHRI controller \cite{Zacharaki2020}, referred to as the ``nominal" controller, to be implemented inside the safe operating region, while respecting the velocity constraints. Thus despite any dangerous behaviour of the nominal control law, the system will remain inside the safe operating region and satisfy velocity constraints. Finally, in the event that a human or environment pushes the robot outside of the safe operating region, the control law will apply a passive restoring force to return the robot to the safe operating region. Our aim is to design a control law to realize this generalized, robust, and passive concept of safety.

To address this, we consider existing ``safety-critical" controllers \cite{Ames2019, Rauscher2016, Hsu2015, wences2020correct,Singletary2020}. Safety-critical control prioritizes safety, while attempting to implement a nominal control. Those methods are based on control barrier functions, which are shown to be more general, rigorously proven methods compared to artificial potential fields, which are common in existing pHRI techniques \cite{Singletary2020a, Zacharaki2020}. Although many safety-critical controllers exist, they do not satisfy the concept of ``safety" considered here. ``Reciprocal" type barriers are not well-defined outside of the operating region and may require exceedingly high (possibly unbounded) control actions to ensure safety \cite{Rauscher2016,Hsu2015}. Methods based on the more recent zeroing control barrier function formulation do not ensure passivity outside the operating region \cite{Ames2019,wences2020correct,Singletary2020}. We note a promising energy-based form of zeroing control barrier functions from \cite{Singletary2020}, however that method is not applicable here as it is not well-defined everywhere outside of the operating region and does not address passivity/robustness. 

In this paper, we develop a novel, closed-form control strategy to ensure safety of a mechanical system in the presence of a human. In contrast to \cite{Singletary2020}, the proposed approach exploits the energy-based barrier function to ensure passivity, robustness, and satisfaction of velocity constraints. The approach also admits any nominal control law \cite{Zacharaki2020} in a \textit{pre-defined} subset of the operating region. Our technique applies to joint and task space operating regions, and behaves well despite singularities that may be encountered in the Jacobian of the forward kinematics. The results are implemented on a 6-DOF robotic arm during a pHRI task.
 
\textit{Notation}: The inequality $A\leq B$ for square matrices $A$ and $B$ means that $B-A$ is positive semi-definite. The interior and boundary of a set $\myset{A}$ are denoted $\mysetint{A}$ and $\partial \myset{A}$, respectively. An extended class-$\mathcal{K}$ function $\alpha: (-b, a) \to \mathbb{R}$, for $a,b \in \mathbb{R}_{>0}$, is a continuous function which is strictly increasing and $\alpha(0) = 0$. The Euclidean norm is denoted by $\| \cdot \|_2$. We say a uniformly continuous function $\myvar{x}:\mathbb{R}_{\geq 0} \to \mathbb{R}^n$ asymptotically approaches a set $\myset{X} \subset \mathbb{R}^n$, if as $t \to \infty$, for each $\varepsilon \in \mathbb{R}_{>0}$, $\exists T \in \mathbb{R}_{>0}$, such that $\text{dist}(\myvar{x}(t), \myset{X}) < \varepsilon \ \forall t\geq T$, where $\text{dist}(\myvar{x}, \myset{X}) := \underset{\myvar{z} \in \myset{X}}{\inf} \| \myvar{x} - \myvar{z}\|$.

\section{Preliminaries}
\subsection{System Dynamics}

Consider the following mechanical system:
\begin{equation}\label{eq:nonlinear affine dynamics}
\begin{split}
\myvardot{q} &= \myvar{v}  \\
\myvardot{v} &= M(\myvar{q})^{-1}(-C(\myvar{q}, \myvar{v}) \myvar{v} - F \myvar{v} - \myvar{g}(\myvar{q}) +  \myvar{u}  )
\end{split}
\end{equation}
where $M(\myvar{q}) \in \mathbb{R}^{n \times n}$ is the inertia matrix, $C(\myvar{q}, \myvar{v}) \in \mathbb{R}^{n\times n}$ is the Coriolis and centrifugal matrix, $\myvar{g}(\myvar{q}) \in \mathbb{R}^{n}$ is the generalized gravity on the system, $F \in \mathbb{R}^{n \times n}$ is the positive definite damping matrix, and  $\myvar{u} \in \mathbb{R}^m$ is the control input for the general joint and velocity states $\myvar{q}, \myvar{v} \in \mathbb{R}^n$, respectively. Let $(\myvar{q}(t, \myvar{q}_0), \myvar{v}(t,\myvar{v}_0)) \in \mathbb{R}^{2n}$ be the solution of \eqref{eq:nonlinear affine dynamics} starting at $t = 0$, which for ease of notation is denoted by $(\myvar{q}, \myvar{v})$. 

Here we consider the following well-known properties for mechanical systems \cite{Spong1989}:
\begin{property}:\label{prop:M}
$M(\myvar{q})$ is symmetric and positive-definite such that there exists $\mu_1 \in \mathbb{R}_{>0}$, with $\mu_1 I_{n\times n} \leq M(\myvar{q}), \ \forall \myvar{q} \in \mathbb{R}^n$. 
\end{property}

\begin{property}: \label{prop:skew symmetric}
$\myvar{x}^T \left( \dot{M}(\myvar{q}) - 2 C(\myvar{q}, \myvar{v}) \right) \myvar{x} = 0$, $\forall \myvar{x} \in \mathbb{R}^n$.
\end{property}

\subsection{Problem Formulation}

Let $c: \mathbb{R}^n \to \mathbb{R}$ be a continuously differentiable function that encodes the constraint set defined as:
\begin{equation}\label{eq:Q set}
\myset{Q} = \{ \myvar{q} \in \mathbb{R}^n: c(\myvar{q}) \geq 0\}
\end{equation}
The constraint function $c(\myvar{q})$ represents any position-based constraint, written in either joint-space or task-space variables. Here we address the case when $\myset{Q}$ is compact, but can otherwise be convex or non-convex.  Examples include  a region bounded away from a human, a region of non-singular configurations, or a pre-defined workspace for the robot. 

We define the velocity constraint set as:
\begin{equation}\label{eq:V set}
    \myset{V} = \{ \myvar{q} \in \mathbb{R}^n: \| \myvar{v} \|_2^2 \leq \bar{v} \}
\end{equation}
for some maximum velocity bound $\bar{v} \in \mathbb{R}_{>0}$.

We now formally define the problem of designing a control law to enforce safety for mechanical systems:
\begin{problem}
Given the system \eqref{eq:nonlinear affine dynamics}, a nominal control law $\myvar{u}_{nom}: \mathbb{R}^n \times \mathbb{R}^n \times \mathbb{R}_{\geq 0} \to \mathbb{R}^n $, and the compact, non-empty constraint sets $\myset{Q}, \myset{V}$ defined by \eqref{eq:Q set} and \eqref{eq:V set}, define a control law $\myvar{u}$ that ensures:
\begin{enumerate}
    \item if $(\myvar{q}(0), \myvar{v}(0)) \in \myset{Q} \times \myset{V}$, then $(\myvar{q}(t), \myvar{v}(t))$ remains in $\myset{Q} \times \myset{V}$ for all $t \geq 0$
    
    \item $\myvar{u} = \myvar{u}_{nom}$ in a pre-defined subset of $\myset{Q} \times \myset{V}$. 
    
    \item if $(\myvar{q}(0), \myvar{v}(0)) \notin \myset{Q} \times \myset{V}$, then the system is passive and furthermore $(\myvar{q}(t), \myvar{v}(t))$ asymptotically approaches $\myset{Q} \times \myset{V}$. 
\end{enumerate} 
\end{problem}

\section{Proposed Solution}

\subsection{Background}\label{ssec:background}

Here we introduce the existing work on the energy-based barrier function from \cite{Singletary2020} using the notation here: 

\begin{definition}[\cite{Singletary2020}]
Given a kinematic safety constraint expressed as a function $c: \myset{Q} \subset \mathbb{R}^n \to \mathbb{R}$, only dependent on $\myvar{q}$, and the corresponding safe set $\myset{S} = \{ (\myvar{q}, \myvar{v}) \in \myset{Q} \times \mathbb{R}^n: c(\myvar{q}) \geq 0\}$, the associated energy-based safety constraint is defined as:
\begin{equation}\label{eq:ZCBF h}
    h(\myvar{q},\myvar{v}) = k_h c(\myvar{q}) - \frac{1}{2}\myvar{v}^T M(\myvar{q}) \myvar{v} 
\end{equation}
with $k_h \in \mathbb{R}_{>0}$. The corresponding energy-based safe set is: $S_D := \{(\myvar{q}, \myvar{v}) \in \myset{Q} \times \mathbb{R}^n: h(\myvar{q}, \myvar{v}) \geq 0\}$.
\end{definition}

Forward invariance of $\myset{S}_D$ is ensured in Theorem 2 of \cite{Singletary2020} under the following control law:

\begin{align}\label{eq:zcbf kinematic qp}
\begin{split}
\myvar{u}^*(\myvar{q}, \myvar{v}) \hspace{0.1cm} = \hspace{0.1cm} & \underset{\myvar{u} \in \mathbb{R}^m} {\text{argmin}}
\hspace{.3cm} \| \myvar{u} -\myvar{u}_{\text{nom}}(\myvar{q}, \myvar{v}, t) \|^2_2  \\
& \text{s.t.} \hspace{.1cm}  \myvar{v}^T ( k_h \nabla c(\myvar{q})  + \myvar{g}(\myvar{q}) - \myvar{u})  \geq - \alpha(h(\myvar{q}, \myvar{v})) 
\end{split}
\end{align}
where $\alpha$ is an extended class-$\mathcal{K}$ function.

There are a few important aspects to note regarding the control law \eqref{eq:zcbf kinematic qp}. The control formulation is dependent on the zeroing control barrier function formulation from \cite{Ames2019}. Under ideal conditions, this control should work well to ensure safety. However the control law is not well defined outside of $\myset{S}_D$. To see this, consider any point for which $h < 0$ (i.e. the system is outside the set $\myset{S}_D$), where we note that $-\alpha(h) > 0$. Then the constraint in \eqref{eq:zcbf kinematic qp} can never be satisfied for a static system, i.e., when $\myvar{v} = 0$. Furthermore as $\|\myvar{v}\| \to 0$, $\|\myvar{u} \| \to \infty$ in an attempt to satisfy \eqref{eq:zcbf kinematic qp}. This ill-posedness is noted in Definition 1 of \cite{Singletary2020} where the barrier condition does not need to hold everywhere outside of  $\myset{S}_D$. Should a perturbation arise (e.g., from human-robot interaction) then the control law may in fact be dangerous in the presence of a human. Here we extend the approach from \cite{Singletary2020} to ensure safety, passivity, and robustness of mechanical systems in the presence of humans.  

\subsection{Passivity-based Set-Invariance Control for Mechanical Systems}

In this section, we address safety, robustness, and passivity of the mechanical system with respect to a safe operating set. The idea here is to exploit properties of mechanical systems, namely the skew-symmetric Property \ref{prop:skew symmetric} to ensure safety and passivity. Recall the function \eqref{eq:ZCBF h}. We will refer to this function as the ``energy-based barrier function." We similarly define the ``safe set" as:
\begin{equation}\label{eq:ZCBF C}
    \myset{C}= \{ (\myvar{q},\myvar{v}) \in \mathbb{R}^n \times \mathbb{R}^n: h(\myvar{q},\myvar{v}) \geq 0 \}
\end{equation}

Our first task to comply with standard safety specifications is to satisfy velocity constraints, i.e., $\myvar{v} \in \myset{V}$. To do so, first we inspect the energy-based barrier function \eqref{eq:ZCBF h}.
The velocity term in $h$ acts to moderate the speed at which the system approaches the boundary of $\myset{Q}$. The velocity is moderated by the inertia matrix, which intuitively means that for systems with large inertia, the system will quickly approach the boundary of the safe set $\myset{C}$ (i.e. $h = 0$). On the other hand, systems with low inertia can approach the boundary at higher speeds with milder consequences as they can be slowed down more easily. 

We can moderate the speed of the system in $\myset{C}$ by tuning $k_h$. The reason for this is that $h \geq 0$ implies that $\frac{1}{2} \myvar{v}^T M(\myvar{q}) \myvar{v} \leq k_h c(\myvar{q})$. For a compact $\myset{Q}$, it is straightforward to see that $h(\myvar{q}, \myvar{v}) \geq 0$ ensures bounded velocities in $\myset{C}$. In the following Lemma we show that for certain $k_h$, if the state $(\myvar{q}, \myvar{v})$ remains in $\myset{C}$, then the state also remains in the constraint sets $\myset{Q}$ and $\myset{V}$: 

\begin{lemma}\label{lem:compact C}
Consider the system \eqref{eq:nonlinear affine dynamics} with constraint sets \eqref{eq:Q set} and \eqref{eq:V set}. If $\myset{Q}$ is compact, then the set $\myset{C}$ is also compact. If additionally $k_h \leq \dfrac{\mu_1 \bar{v}}{2 \bar{c}}$ for $\bar{c} = \max_{\myvar{q} \in \myset{Q}} c(\myvar{q})$, then $\myset{C} \subset \myset{Q} \times \myset{V}$.
\end{lemma}
\begin{proof}
Consider $\myset{C}$. For $h(\myvar{q},\myvar{v}) \geq 0$ it follows that $\frac{1}{2}\myvar{v}^T M(\myvar{q}) \myvar{v} \leq k_h c(\myvar{q})$. Since $M(\myvar{q})$ is positive-definite and $c(\myvar{q})$ is bounded in $\myset{Q}$, it follows that $\myvar{v}^T M(\myvar{q}) \myvar{v}$ is bounded and thus $\myset{C}$ is compact.

Now we show that if $(\myvar{q},\myvar{v}) \in \myset{C}$, then $\myvar{q} \in \myset{Q}$. We note that since $M(\myvar{q})$ is positive-definite $\myvar{v}^T M(q) \myvar{v} \geq 0$. Now for $(\myvar{q},\myvar{v}) \in \myset{C}$, $h \geq 0$ implies that $k_h c(\myvar{q}) \geq \myvar{v}^T M(\myvar{q})v \geq 0$, and thus $\myvar{q} \in \myset{Q}$. 

Finally, by Property \ref{prop:M} and the fact that $\frac{1}{2}\myvar{v}^T M(\myvar{q}) \myvar{v} \leq k_h c(\myvar{q}) \leq k_h \bar{c}$ for $(\myvar{q}, \myvar{v}) \in \myset{C}$, it follows that $\frac{\mu_1}{2} \| \myvar{v} \|_2^2 \leq \frac{1}{2}\myvar{v}^T M(\myvar{q}) \myvar{v} \leq k_h \bar{c}$. Solving for $\myvar{v}$ yields: $ \| \myvar{v} \|_2^2 \leq \frac{2 k_h \bar{c}}{\mu_1}$, and substitution of $k_h \leq \frac{\mu_1 \bar{v}}{2 \bar{c}}$ yields $\| \myvar{v} \|_2^2 \leq \bar{v}$. Thus, if $(\myvar{q}, \myvar{v}) \in \myset{C}$, then $\myvar{v} \in \myset{V}$. Since for any $(\myvar{q}, \myvar{v}) \in \myset{C}$, $(\myvar{q}, \myvar{v}) \in \myset{Q} \times \myset{V}$, it follows that $\myset{C} \subset \myset{Q} \times \myset{V}$.
\end{proof}

Lemma \ref{lem:compact C} ensures that $h$, is properly defined such that the task of ensuring safety i.e. $(\myvar{q}, \myvar{v}) \in \myset{Q}\times \myset{V}$, boils down to ensuring $\myset{C}$ can be rendered forward invariant.

To show forward invariance of $\myset{C}$, we depart from the ``zeroing control barrier function" framework from \cite{Ames2019, Singletary2020}. Instead, we exploit the fundamentals of set invariance control (see Brezis' theorem \cite{Redheffer1972}), which later allows us to ensure not only passivity, but also robustness of the safe set. 

Brezis' theorem states that to ensure forward invariance of a set, the vector field of the dynamical system at the boundary of the  set must be inside the tangent cone of the set. We re-write this condition in our notation as follows:
\begin{equation}\label{eq:brezi cond}
    \dot{h}(\myvar{q}, \myvar{v}) \geq 0, \ \forall (\myvar{q}, \myvar{v}) \in \mysetbound{C}
\end{equation}
We note that for the above condition to hold, we require that the closed-loop dynamics are locally Lipschitz in an open set $\myset{D}$ containing $\myset{C}$. Our next task is thus to construct a locally Lipschitz $\myvar{u}$ such that \eqref{eq:nonlinear affine dynamics} under $\myvar{u}$ satisfies \eqref{eq:brezi cond}. To do so, we continue by differentiating $h$ which yields:
\begin{align*}
    \dot{h} & = -\myvar{v}^T M \myvardot{v} - \frac{1}{2} \myvar{v}^T \dot{M} \myvar{v} + k_h \nabla c(\myvar{q})^T \myvar{v} \\
    & = -\myvar{v}^T \left( -C \myvar{v} -F\myvar{v}- \myvar{g} + \myvar{u} \right) - \frac{1}{2} \myvar{v}^T \dot{M} \myvar{v} + k_h \nabla c(\myvar{q})^T \myvar{v} \\
    & = \myvar{v}^T (k_h \nabla c(\myvar{q}) + \myvar{g} - \myvar{u} + F \myvar{v})
\end{align*}
Note that Property \ref{prop:skew symmetric} i.e. $\myvar{v}^T(\frac{1}{2}\dot{M} - C) \myvar{v} = 0$, is used in the above sequence of equations. Furthermore, since $F$ is positive definite, it follows that:
\begin{equation}\label{eq:hdot inequality}
 \dot{h} \geq \myvar{v}^T(k_h \nabla c(\myvar{q}) + \myvar{g} - \myvar{u})
\end{equation}

Next, we define the point-wise set $\myset{K}_u(q,v)$ for which any $\myvar{u}(\myvar{q},\myvar{v}) \in \myset{K}_u(\myvar{q},\myvar{v})$ will ensure forward invariance of $\myset{C}$:
\begin{align}\label{eq:Ku set}
\myset{K}_u(\myvar{q},\myvar{v}) = \{ \myvar{u} \in \mathbb{R}^n: \myvar{v}^T \left( k_h \nabla c(\myvar{q})+ \myvar{g}(\myvar{q}) - \myvar{u} \right)  
 \geq 0 \}
\end{align}

We are now ready to guarantee forward invariance of $\myset{C}$:
\begin{thm}\label{thm:zcbf}
Suppose $\myset{Q}$ defined by \eqref{eq:Q set} is compact for a continuously differentiable function $c: \mathbb{R}^n \to \mathbb{R}$. Let $\myset{D} \subset \mathbb{R}^n \times \mathbb{R}^n$ be any bounded, open set containing $\myset{C}$ such that $\myset{C} \subset \myset{D}$ for $\myset{C}$ defined by \eqref{eq:ZCBF C} and $h$ defined by \eqref{eq:ZCBF h}. If $\nabla c(\myvar{q})$ is locally Lipschitz for all $(\myvar{q},\myvar{v}) \in \myset{D}$, then there always exists a locally Lipschitz $\myvar{u} \in \myset{K}_u(\myvar{q}, \myvar{v})$ for all $(\myvar{q}, \myvar{v}) \in \myset{D}$. Furthermore, for any $\myvar{u} $ that is locally Lipschitz on $\myset{D}$ and satisfies: $\myvar{u} \in \myset{K}_u(\myvar{q}, \myvar{v})$ for all $(\myvar{q}, \myvar{v}) \in \mysetbound{C}$, if $\myvar{q}(0), \myvar{v}(0)) \in \myset{C}$, then \eqref{eq:nonlinear affine dynamics} under $\myvar{u}$ ensures $(\myvar{q}(t), \myvar{v}(t)) \in \myset{C}$ for all $t\geq 0$.  
\end{thm}
\begin{proof}
Differentiation of $h(\myvar{q},\myvar{v})$ yields \eqref{eq:hdot inequality}. It is clear that the choice of $\myvar{u} = \myvar{g} + k_h \nabla c(\myvar{q})$ is locally Lipschitz and always a member of the set $\myset{K}_u(\myvar{q},\myvar{v})$. Furthermore, $\myvar{u}$ is well-defined for all $(\myvar{q}, \myvar{v}) \in \myset{D}$.

For any $\myvar{u} \in \myset{K}_u(\myvar{q}, \myvar{v})$ locally Lipschitz on $\myset{D}$, it follows that the closed-loop system \eqref{eq:nonlinear affine dynamics} under this $\myvar{u}$ is locally Lipschitz for all $(\myvar{q}, \myvar{v}) \in \myset{D}$ and also $\dot{h} \geq 0$ for any $(\myvar{q}, \myvar{v}) \in \mysetbound{C}$. Since $\myvar{q}(0), \myvar{v}(0) \in \myset{C}$, the unique, uniformly continuous solution $(\myvar{q}(t), \myvar{v}(t))$ exists for $t\in [0, T)$, for some $T \in \mathbb{R}_{>0}$ via Theorem 3.1 of \cite{Khalil2002}, and so Brezis' theorem ensures that $(\myvar{q}(t), \myvar{v}(t)) \in \myset{C}$ for $t \in [0, T)$ \cite{Redheffer1972}. 

Now since $\myset{C}$ is compact, the state $(\myvar{q}, \myvar{v})$ will never leave $\myset{D}$ for which local Lipschitz properties of the closed-loop system hold. Thus we can repeat the previous analysis ad infinitum and extend $T \to \infty$, and so $(\myvar{q}(t), \myvar{v}(t)) \in \myset{C}$ for all $t \geq 0$.
\end{proof}
We note that Theorem \ref{thm:zcbf} requires local Lipschitz continuity of $\nabla c$. In practice, this is not restrictive as in many cases $c(\myvar{q})$ is a twice-continuously differentiable function and so $\nabla c$ is locally Lipschitz. A common example of $c(\myvar{q})$ satisfying this local Lipschitz condition is the spherical/ellipsoidal operating region  (see Section \ref{sec:HRI exp}).

With the analysis herein, we can take the results of Theorem \ref{thm:zcbf} one step further by ensuring robustness and then \textit{passivity} of the system outside of $\myset{C}$. We list these results in order of increasing conditions required on $\myvar{u}$. First we state robustness results:
\begin{thm}\label{thm:robustness}
Suppose the conditions of Theorem \ref{thm:zcbf} hold, and  for all $(\myvar{q}, \myvar{v}) \in \myset{D} \setminus \myset{C}$, $\myvar{u} \in \myset{K}_u(\myvar{q},\myvar{v})$ and $\myvar{u} \neq \myvar{g}(\myvar{q})$. Then for all $(\myvar{q}, \myvar{v}) \in \myset{D}\setminus \myset{C}$, $(\myvar{q}(t), \myvar{v}(t))$ asymptotically approaches $\myset{C}$.
\end{thm}
\begin{proof}
From Theorem \ref{thm:zcbf}, it follows that $\myset{C}$ is forward invariant and furthermore $(\myvar{q}(t), \myvar{v}(t))$ is forward complete in $\myset{D}$. Consider the following continuously differentiable Lyapunov-like function:
\begin{align*}
    V(\myvar{q}, \myvar{v}) = \begin{cases}
    -h(\myvar{q}, \myvar{v}), \text{ if } h(\myvar{q}, \myvar{v}) \leq -1 \\
    h(\myvar{q}, \myvar{v})^3 + 2h(\myvar{q}, \myvar{v})^2, \text{ if } -1 \leq h(\myvar{q}, \myvar{v}) \leq 0 \\
    0, \text{ if } h(\myvar{q}, \myvar{v}) \geq 0 \end{cases}
\end{align*}

We differentiate $V$ with respect to the three cases as follows. First, for $h \leq -1$ and from the definition of $\myset{K}_u$: $\dot{V} = -\dot{h} = -\myvar{v}^T F \myvar{v} - \myvar{v}^T(k_h \nabla c(\myvar{q}) + \myvar{g} - \myvar{u}) \leq 0$. Next, when $h \in [-1, 0]$, $\dot{V} = (3h^2 + 4h)\dot{h}$. Since for $h \in [-1,0]$, $3h^2 + 4h \leq 0$ and $\dot{h} \geq 0$ from $\myvar{u} \in \myset{K}_u$, it follows that $\dot{V} \leq 0$. Finally for $h \geq 0$, i.e. in $\myset{C}$, $\dot{V} = 0$. Thus by construction, $\dot{V} \leq 0$ in $\myset{D} \setminus \myset{C}$ and $\dot{V} = 0$ in $\myset{C}$. Thus $\dot{V}$ is negative semi-definite and $V$ is decreasing or possibly constant in $\myset{D} \setminus \myset{C}$. 

Also, we claim that in $\myset{D} \setminus \myset{C}$ (i.e when $h < 0$), $\dot{V} = 0$ if and only if $\myvar{v} = 0$. To see this, we note that when $h \leq -1$, $\myvar{v} = 0$ implies $\dot{V} = 0$. Now if $\dot{V} = 0$, then $-\myvar{v}^T F \myvar{v} - \myvar{v}^T(k_h \nabla c(\myvar{q}) + \myvar{g} - \myvar{u}) = 0$. Since $\myvar{u} \in \myset{K}_u$, then both terms of $\dot{V}$ are non-positive. Thus for $\dot{V} = 0$, then $\myvar{v} = 0$ must hold since $F$ is positive-definite. Similarly when $h \in [-1, 0)$, $\dot{V} = (3h^2 + 4h) \dot{h}$ for which $3h^2 + 4h < 0$. Thus using the same argument for when $h \leq -1$, we see that $\dot{V} = 0$ if and only if $\myvar{v} = 0$ for $h \in [-1, 0)$, and the claim holds for $h < 0$ i.e. in $\myset{D} \setminus \myset{C}$. Thus it follows that the set $\myset{A} = \{(\myvar{q}, \myvar{v} \in \myset{D}\setminus \myset{C}: \dot{V} = 0\} $ is equivalent to the set $\myset{B} = \{(\myvar{q}, \myvar{v}) \in \myset{D} \setminus \myset{C}: \| \myvar{v} \| = 0 \}$. 

Next we show that no solution of \eqref{eq:nonlinear affine dynamics} can stay identically in the set $\myset{A}$. We prove this by contradiction. Since the set $\myset{A}$ is equivalent to $\myset{B}$, we suppose there is a stationary state, i.e., $\myvar{v} \equiv 0$ for $(\myvar{q}, \myvar{v}) \in \myset{D} \setminus \myset{C}$. Since $\myvar{v} \equiv 0$,  $\myvardot{v} = 0$ and substitution into \eqref{eq:nonlinear affine dynamics} yields: $0 = M^{-1} (\myvar{u} - \myvar{g}) = \myvar{u} - \myvar{g} \implies \myvar{u} = \myvar{g}$. However, by assumption $\myvar{u} \neq \myvar{g}$. Thus a contradiction and so no solution can stay identically in $\myset{A}$. 

Finally, since $\myset{C}$ is forward invariant from Theorem \ref{thm:zcbf} and the solution $(\myvar{q}(t), \myvar{v}(t))$ is bounded via Lemma \ref{lem:compact C}, there exists a non-empty, largest invariant set $M \subset \myset{C}$ via Lemma 4.1 of \cite{Khalil2002}. Now since no solution can stay in $\myset{A}$, then the largest invariant set contained in $\myset{D}$, is $\myset{M} \subset \myset{C}$. From Theorem 4.4 of \cite{Khalil2002}, this implies that for any $(\myvar{q}(t), \myvar{v}(t)) \in \myset{D} \setminus \myset{C}$, $(\myvar{q}(t), \myvar{v}(t))$ asymptotically approaches $\myset{M} \subset \myset{C}$. Thus the solution  asymptotically approaches $\myset{C}$.
\end{proof}

An immediate result of Theorem \ref{thm:robustness} is that the system is robust to perturbations either from the environment or from model uncertainty. To incorporate model uncertainties into the safety-critical control, we further shrink the set $\myset{C}$ using a robustness margin that is dependent on the upper bound of the perturbation. We refer to \cite{Xu2015a} for further reference on how this can be done.

To state passivity results, we consider the following system dynamics with an exogenous input to the system, $\myvar{\mu} \in \mathbb{R}^n$, which can represent a disturbance from the human or the environment:
\begin{equation}\label{eq:nonlinear affine dynamics input}
\begin{split}
\myvardot{q} &= \myvar{v}  \\
\myvardot{v} &= M(\myvar{q})^{-1}(-C(\myvar{q}, \myvar{v}) \myvar{v} - F \myvar{v} - \myvar{g}(\myvar{q}) +  \myvar{u} + \myvar{\mu} )
\end{split}
\end{equation}

\begin{cor}\label{thm:passivity} 
Suppose the conditions of Theorem \ref{thm:zcbf} hold for the system \eqref{eq:nonlinear affine dynamics input} for any $\myvar{\mu} \in \mathbb{R}^n$. If $\myvar{u} = \myvar{g}(\myvar{q}) + k_h \nabla c(\myvar{q})$ in $\myset{D}\setminus \mysetint{C}$, then the system is passive with respect to $\myvar{\mu}$ in $\myset{D} \setminus \mysetint{C}$. Furthermore, if $\nabla c(\myvar{q}) \neq 0$ in $\myset{D} \setminus \myset{C}$ and $\myvar{\mu} \equiv 0$, then also $(\myvar{q}(t), \myvar{v}(t))$ asymptotically approaches $\myset{C}$ in $\myset{D} \setminus \myset{C}$.
\end{cor}

\begin{proof}
Consider the storage function $S(\myvar{q}, \myvar{v}) = -h(\myvar{q}, \myvar{v})$, and note that $S$ is non-negative in $\myset{D} \setminus \mysetint{C}$. Differentiation of $S$ yields: $\dot{S} \leq -\myvar{v}^T (k_h \nabla c(\myvar{q}) + \myvar{g} - \myvar{u} - \myvar{\mu})$.
Substitution of $\myvar{u} = \myvar{g}(\myvar{q}) + k_h \nabla c(\myvar{q})$ yields: $\dot{S} \leq \myvar{v}^T \myvar{\mu}$. We consider $\myvar{\mu}$ as the input and $\myvar{y} = \myvar{v}$ the output of the system \eqref{eq:nonlinear affine dynamics input}. It follows that the system is passive (see Definition 3.1.4 of \cite{Schaft2017}) in $\myset{D} \setminus \mysetint{C}$.

Furthermore, if $\nabla c(\myvar{q}) \neq 0$ then clearly $\myvar{u} \neq \myvar{g}(\myvar{q}) $ in $\myset{D} \setminus \myset{C}$. So if additionally $\myvar{\mu} \equiv 0$, we recover the result from Theorem \ref{thm:robustness}. Note that the condition that $\nabla c(\myvar{q}) \neq 0$ in $\myset{D}\setminus \myset{C}$ also ensures zero-state observability in $\myset{D} \setminus \myset{C}$.
\end{proof}
 In Corollary \ref{thm:passivity}, the condition $\nabla c \neq 0$ in $\myset{D} \setminus \myset{C}$ is required for robustness to hold, although \textit{passivity} will hold regardless if $\nabla c = 0$ or not. In many practical cases, e.g., ellipsoidal operating regions defined in joint space, $\nabla c = 0$ only occurs inside $\myset{Q}$ and so robustness is preserved. In non-convex sets, possibly including task-space constraints, this may not hold and may need to be checked a priori (see Section \ref{sec:HRI exp}). 

\subsection{Control Design}

Theorems \ref{thm:zcbf}, \ref{thm:robustness}, and Corollary \ref{thm:passivity} highlight key requirements that any control law should satisfy to ensure forward invariance, robustness, and passivity of the mechanical system, i.e., safety. We proceed by defining a safety-critical control law, $\myvar{u}$, that satisfies these requirements, while also admitting an existing nominal control law $\myvar{u}_{nom}: \mathbb{R}^n \times \mathbb{R}^n \times \mathbb{R}_{\geq 0} \to \mathbb{R}^n$ to be applied inside $\myset{C}$. We define this control as follows:
\begin{equation}\label{eq:safety controller}
    \myvar{u} =   \left( 1- \phi_{\varepsilon}(h) \right) \left(\myvar{g}(\myvar{q}) + k_h \nabla c(\myvar{q}) \right) + \phi_{\varepsilon}(h) \myvar{u}_{nom}(\myvar{q}, \myvar{v},t)
\end{equation}
where $\phi_{\varepsilon}: \mathbb{R} \to [0, 1]$ is defined by:
\begin{align} \label{eq:phi}
    \phi_{\varepsilon}(h) = \begin{cases}
    1, \text{ if } h > \varepsilon \\
    \kappa(h), \text{ if } h \in [0,\varepsilon] \\
    0, \text{ if } h < 0
    \end{cases}
\end{align}
and $\kappa: \mathbb{R} \to [0,1]$ is any locally Lipschitz continuous function that satisfies $\kappa(0) = 0$ and $\kappa(\varepsilon) = 1$, for some $\varepsilon \in \mathbb{R}_{>0}$. The design parameters $\kappa$ and $\varepsilon$ tune how ``aggressive" the system behaves to ensure safety. 

The role of $\phi_{\varepsilon}$ is to define how the control law \eqref{eq:safety controller} transitions between $\myvar{u}_{nom}$ and a safe control law that satisfies Theorems \ref{thm:zcbf}, \ref{thm:robustness} and Corollary \ref{thm:passivity}. This is less conservative than the original methods from \cite{Ames2019, Singletary2020} because we only enforce $\dot{h} \geq 0$ for $h \leq 0$. This is done via $\phi_\varepsilon(h)$. In \cite{Ames2019, Singletary2020}, the zeroing barrier function formulation requires $\dot{h} \geq - \alpha(h)$ for all $h$, which restricts what control can be implemented in $\myset{D}$. Furthermore, the proposed control is well-defined in $\myset{D}$, whereas the control \eqref{eq:zcbf kinematic qp} is not (see Section \ref{ssec:background}), and so cannot provide any guarantees of robustness or passivity. To see this, we step through each region of $\myset{D}$. First, in the region of $\myset{C}$ for which $h \geq \varepsilon$, $\phi_\varepsilon = 1$ so that $\myvar{u} = \myvar{u}_{nom}$. Thus the designer knows a priori when $\myvar{u}_{nom}$ will be implemented, and we define this set as:
\begin{equation}
    \myset{C}_\varepsilon:=\{ (\myvar{q}, \myvar{v}) \in \mathbb{R}^n \times \mathbb{R}^n: h(\myvar{q}, \myvar{v}) \geq \varepsilon\} \subset \myset{C}
\end{equation}
This design ensures that the original task, whether it be teleoperation with a human-in-the-loop, a learning-based task, or a stabilizing controller, will \textit{always} be implemented to ensure performance is unchanged in $\myset{C}_\varepsilon$. The term $\varepsilon$ is a tuning parameter that can be adjusted by the designer. In theory, $\varepsilon$ can be arbitrarily small, however this would result in a quick transition between $\myvar{u}_{nom}$ and $\myvar{g} + k_h \nabla c$, which may cause issues in implementation with respect to noise and low sampling rates.

Next, in the region between $\myset{C}$ and $\myset{C}_\varepsilon$, $\phi_\varepsilon$ transitions between $\myvar{u}_{nom}$ and the safe, passive control: $\myvar{u} = \myvar{g} + k_h \nabla c$. This transition is defined by $\kappa$, which is also freely chosen by the designer so long as it is locally Lipschitz and ensures continuity of \eqref{eq:safety controller}. Finally, at the boundary of $\myset{C}$, where $h = 0$, and for all states outside of $\myset{C}$, $\myvar{u} = \myvar{g} + k_h \nabla c$. From Theorem \ref{thm:zcbf}, we know that this choice of control is always in $\myset{K}_u$ so that forward invariance of $\myset{C}$ is guaranteed. From Corollary \ref{thm:passivity}, this control ensures passivity outside of $\myset{C}$, and if in addition $\nabla c \neq 0$ outside of $\myset{C}$, then the control also ensures robustness in the form of asymptotic stability to $\myset{C}$.

We state the formal guarantees of safety, passivity, and robustness of the proposed control in the following theorem:
\begin{thm}\label{thm:control guarantees}
Suppose $\myset{Q}$ defined by \eqref{eq:Q set} is compact for a continuously differentiable function $c: \mathbb{R}^n \to \mathbb{R}$. Let $\myset{D} \subset \mathbb{R}^n$ be any bounded, open set containing $\myset{C}$ such that $\myset{C} \subset \myset{D}$ for $\myset{C}$ defined by \eqref{eq:ZCBF C} and $h$ defined by \eqref{eq:ZCBF h}. If $\nabla c(\myvar{q})$ and $\myvar{u}_{nom}(\myvar{q}, \myvar{v}, t)$ are locally Lipschitz for all $(\myvar{q},\myvar{v}) \in \myset{D}$, $t \geq 0$,
the system \eqref{eq:nonlinear affine dynamics} under the control \eqref{eq:safety controller} ensures:
\begin{enumerate}
    \item $\myvar{u} = \myvar{u}_{nom}$ in $\myset{C}_\varepsilon$
    \item if $(\myvar{q}(0), \myvar{v}(0)) \in \myset{C}$, then $(\myvar{q}(t), \myvar{v}(t))$ remains in $\myset{C}$ for all $t\geq 0$
    \item the system is passive in $\myset{D}\setminus \mysetint{C}$
    \item if $(\myvar{q}(0), \myvar{v}(0)) \in \myset{D} \setminus \myset{C}$ and $\nabla c(\myvar{q}) \neq 0$ in $\myset{D} \setminus \myset{C}$, $(\myvar{q}(t), \myvar{v}(t))$ asymptotically approaches $\myset{C}$. 
\end{enumerate}
\end{thm}
\begin{proof}
Follows from Theorems \ref{thm:zcbf}, \ref{thm:robustness}, and Corollary \ref{thm:passivity}.
\end{proof}

Theorem \ref{thm:control guarantees} ensures the proposed control is well-posed in that it always exists to ensure safety. Furthermore, since the control is in closed-form, continuous, and defined on a bounded set, it is also bounded. This means the proposed controller can be designed to also satisfy input constraints. One way of doing this is by applying a saturation function on $\myvar{u}_{nom}$ and restricting the maximum allowable $\myvar{u}_{nom}$ with respect to the maximum values of $\myvar{g}$ and $\nabla c$ to satisfy input constraints. This will be a focus of future work.

\section{Human-Robot Interaction-based Experimental Results}\label{sec:HRI exp}

This section is devoted to the experimental verification of the proposed framework on the $6$-DOF Hebi robotic manipulator. 
We define the safety region as a region where the manipulator can safely operate without harming the humans around it, via a task-space constraint. Then we consider two scenarios.
Firstly, we design the nominal input $\myvar{u}_{nom}$ to track a time-varying trajectory that violates this region. We show that by implementing the proposed control with this nominal control via \eqref{eq:safety controller} that the systems stays within the operating region. Secondly, we implement the proposed control as the human interacts with the manipulator. We show that the manipulator is compliantly attempting to stay within the operating region despite disturbances from the human. The commands are sent to the robotic manipulator via a ROS node over Ethernet network at a frequency of 500 Hz. The control algorithms are implemented in C++ environment in a laptop computer system equipped with $15.3$ GB RAM and $12$-core i$7$-$875$0H CPU at $2.2$GHz.

The safe set is designed as follows. Let the position of the manipulator end-effector (here we consider the position of the motor) be given by the forward kinematics $\myvar{x} = \myvar{f}(\myvar{q})$, where $\myvar{f}: \mathbb{R}^n \to \mathbb{R}^3$. The linear Jacobian is $J(\myvar{q}) = \frac{\partial \myvar{f}}{\partial \myvar{q}}$, such that $\myvardot{x} = J(\myvar{q}) \myvar{v}$. 
The safe set is defined via the ellipsoidal constraint $c(\myvar{q}) = 1 - (\myvar{x}(\myvar{q}) - \myvar{x}_0)^\top P (\myvar{x}(\myvar{q}) - \myvar{x}_0)$, where $\myvar{x}_0 = [0.43,-0.12,0.12]^\top$ represents the center of the ellipsoid, and $P = \text{diag}\{[1.78,1.78, 4.95]\}$, which defines the lengths of the semi-axes as $0.75$m (in $x$, $y$) and $0.45$m (in $z$). 
The gradient $\nabla c(\myvar{q})$ in \eqref{eq:safety controller} takes the form $\nabla c(\myvar{q}) = - 2J(\myvar{q})^\top P (\myvar{x}(\myvar{q}) - \myvar{x}_0)$.
For the following demonstrations we use $\kappa(h) = -\frac{2}{\varepsilon^3}h^3 +\frac{3}{\varepsilon^2}h^2$, and $\varepsilon = 0.1$ in \eqref{eq:phi}.

Here we make a few key observations of the controller for this task-space operating region. First, the control law \eqref{eq:safety controller} is in closed-form and thus can always be defined. It is not a solution to an optimization problem which may have no solution or may fail to find a solution in real-time. Second, we see the influence of the Jacobian in the gradient $\nabla c(\myvar{q})$. Unlike many existing methods that require the inversion of $J(\myvar{q})$ \cite{Zacharaki2020}, the proposed control can be implemented despite any singularities encountered in the workspace. We see this in the extreme case when $\nabla c(\myvar{q}) = 0$. Despite this, Theorem \ref{thm:control guarantees} still ensures forward invariance of $\myset{C}$ and even in the presence of perturbations, the system is still passive in $\myset{D} \setminus \mysetint{C}$. The only, minor, impact of a singularity is that we lose asymptotic stability in regions outside the operating region where $J(\myvar{q})^\top P (\myvar{x}(\myvar{q}) - \myvar{x}_0) = 0$.

\begin{figure}
    \centering
    \includegraphics[width=.3\textwidth]{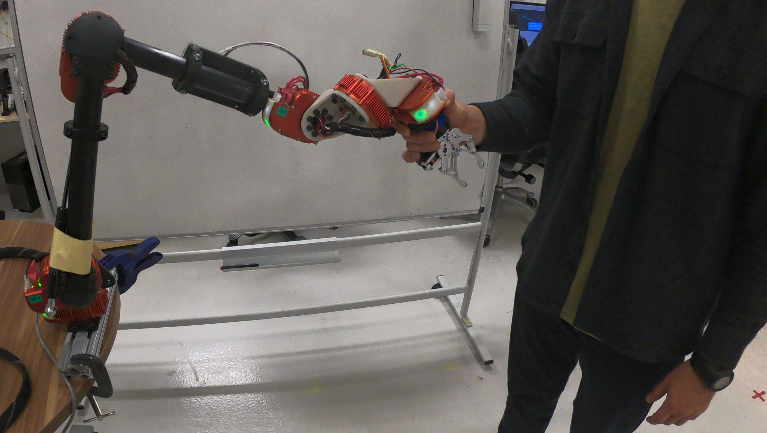}
    \caption{Experimental setup with a human interacting with  the Hebi robot.}
    \label{fig:exp setup}
\end{figure}

\subsection{Safe Implementation of a Nominal Control Law}
We first show how the proposed control behaves when a nominal input $\myvar{u}_{nom}$ attempts to violate the operating region. In particular, we design a standard inverse dynamics controller 
to track a given reference trajectory. 
The top plot of Fig. \ref{fig:c h nominal and barrier} shows that the trajectory associated with the nominal control violates the safe region defined by $c(\myvar{q})$ and $h(\myvar{q},\myvar{v})$. Next, we place a human in the vicinity of the robot, outside of the pre-defined operating region, and we implement the proposed control law \eqref{eq:safety controller} for $k_h = 0.25$. The results seen in the bottom plot of Fig. \ref{fig:c h nominal and barrier} show that both $h$ and $c$ remain non-negative throughout the entire trajectory. We do note that effects of (unmodeled) noise and sampling can be seen in the brief instance near $t = 45$ seconds at which point $h$ is negative. However the robustness of the control law handles the perturbation by pushing the system back into the set $\myset{C}$ as expected. Also, Fig. \ref{fig:inputs_barrier} shows that when $h \geq \varepsilon$ (regions between dashed lines), $\myvar{u} = \myvar{u}_{nom}$ as dictated by the control design. In the accompanying video file \cite{Video}, we see that these results correlate with the robot never leaving the safe operating region to ensure safety of the human.

\begin{figure}[t]
    \centering
    \includegraphics[width=.5\textwidth]{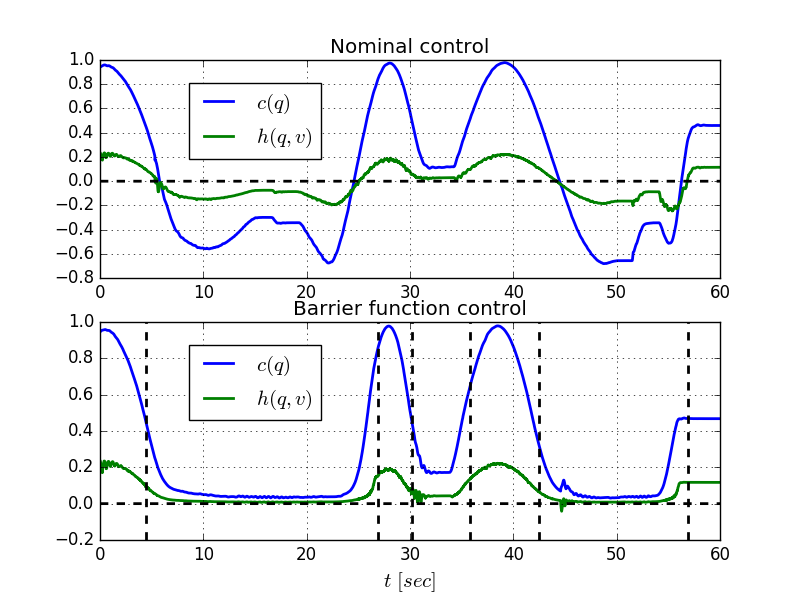}
    \caption{The evolution of $c(\myvar{q}), h(\myvar{q},\myvar{v})$ when $\myvar{u}=\myvar{u}_{nom}$ (top) and the barrier control scheme \eqref{eq:safety controller} (bottom), where vertical dashed lines signify the time instants where ${h}(\myvar{q},\myvar{v}) = \varepsilon=0.1$.}
    \label{fig:c h nominal and barrier}
\end{figure}

\begin{figure}
    \centering
    \includegraphics[width=.5\textwidth]{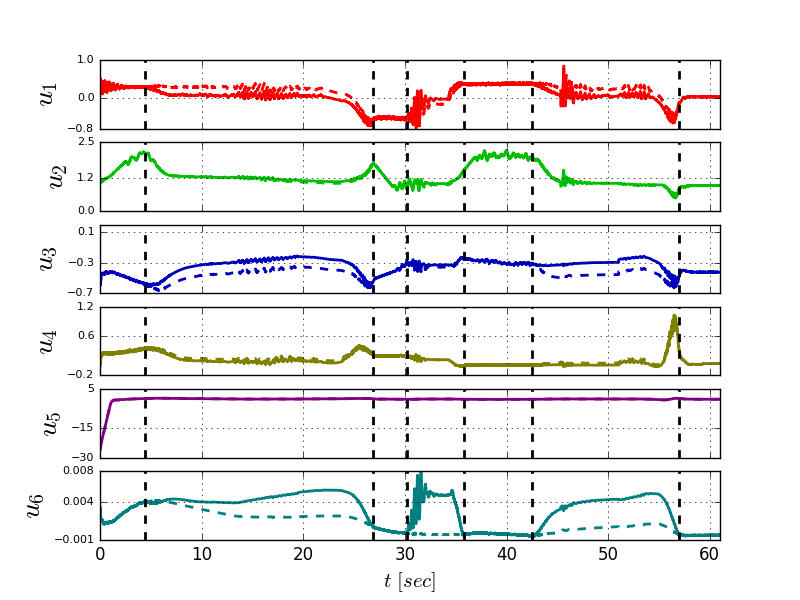}
    \caption{The evolution of the control inputs for the barrier control scheme \eqref{eq:safety controller}, where vertical dashed lines signify the time instants where ${h}(\myvar{q},\myvar{v}) = \varepsilon=0.1$.}
    \label{fig:inputs_barrier}
\end{figure}

\subsection{Passivity in Physical Human-Robot Contact}
We now show how the proposed scheme behaves when the robotic manipulator is in contact with a human. Here we set the nominal control as a gravity compensator, $\myvar{u}_{nom} = g$. The human attempts to violate the set $\mathcal{C}$ both by pushing the robot outside the safe region $\mathcal{Q}$ as well as moving the end-effector at excessive speeds. The results of applying the proposed control law \eqref{eq:safety controller} are show in Fig. \ref{fig:human interaction} and are depicted in the accompanying video file \cite{Video}. Within the first 70 seconds, the human pushes the robot outside of the operating region. This is seen as both $c$ and $h$ become negative at which point the safe control law (i.e. $\myvar{u} = \myvar{g} + k_h \nabla c$ ) passively pushes the system back into the operating region. At about $t =$ 55 seconds, the human pushes the robot into and beyond a singular configuration. Despite the proximity to singularity, the control is still well-defined (see the bottom plot of Fig. \ref{fig:human interaction}), and the robot is still able to return to the safe set. After 70 seconds, the human attempts to push the system to large velocities, while remaining inside the operating region. This is seen as $c$ remains positive while $h$ becomes negative. Here the control acts to resist the excessive speed and dampen out the humans actions. This demonstration shows that the proposed method ``behaves well" in the sense that it is well-defined, passive, and robust to (human) perturbations.

\begin{figure}
    \centering
    \includegraphics[width=.5\textwidth]{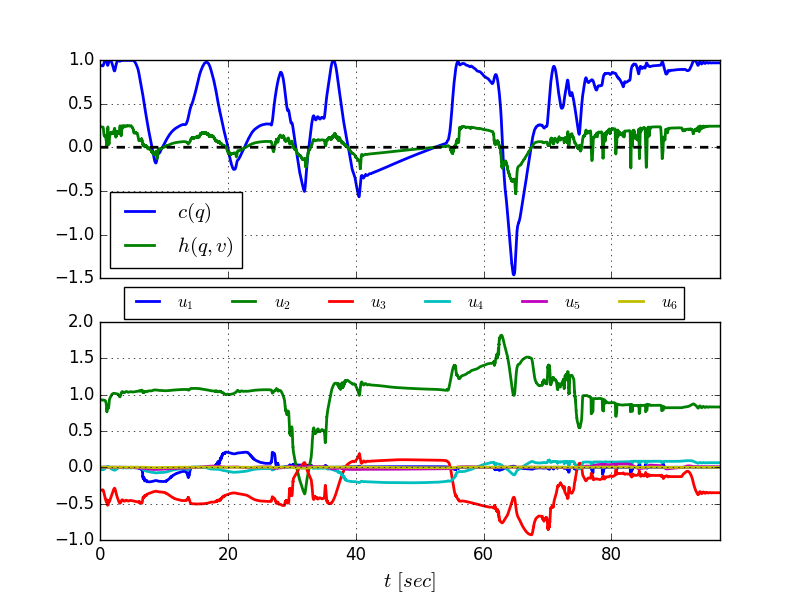}
    \caption{The evolution of ${c}(\myvar{q})$, ${h}(\myvar{q},\myvar{v})$ (top) along with the inputs $\myvar{u}$ (bottom) for human-robot contact experiment.}
    \label{fig:human interaction}
\end{figure}

\section{Conclusion}

In this paper, we developed a safe, passive, and robust control law for mechanical systems. The proposed control is used to ensure forward invariance of a specified operating region and ensures velocity requirements are always respected in the operating region. Furthermore, the control law allows any existing, nominal control law to be implemented in the operating region as long as the safety requirements are adhered to. Finally, the control is well-defined in the robot workspace and ensures robustness and passivity of the system outside the operating region. The results presented include formal guarantees of safety and a demonstration of the proposed method for a task-space based operating region on a 6-DOF robot. Future work will consider extensions to multi-robot and human interactions.

\bibliographystyle{IEEEtran}
\bibliography{IEEEabrv,CBF_mech_conf}

\end{document}